%% file: main.tex
\documentclass{article}
\usepackage[numbers]{natbib}

\usepackage[final]{neurips_2022}

\usepackage[utf8]{inputenc} 
\usepackage[T1]{fontenc}    
\usepackage{hyperref}       
\usepackage{url}            
\usepackage{booktabs}       
\usepackage{amsfonts}       
\usepackage{amsmath}
\usepackage{amssymb}
\usepackage{amsthm}
\usepackage{bbm}
\usepackage{mathtools}
\usepackage{nicefrac}       
\usepackage{microtype}      
\usepackage{xcolor}         
\usepackage{subcaption}     

\newcommand\BM\boldsymbol
\newcommand\BB\mathbb
\newcommand\CAL\mathcal
\newcommand\BAR\widebar
\newcommand\HAT\widehat
\newcommand\TLD\widetilde

\DeclareMathOperator*{\argmin}{argmin}

\definecolor{myRed}{rgb}{1,0,0}
\definecolor{myBlue}{rgb}{0,0,1}

\usepackage{booktabs}
\usepackage{multirow}
\usepackage{enumitem}

\usepackage{algorithm}
\usepackage{algorithmic}

\usepackage{calc}
\newsavebox\CBox
\newcommand\hcancel[2][0.5pt]{%
  \ifmmode\sbox\CBox{$#2$}\else\sbox\CBox{#2}\fi%
  \makebox[0pt][l]{\usebox\CBox}%
  \rule[0.5\ht\CBox-#1/2]{\wd\CBox}{#1}}

\newcommand\method{DIMES}
\title{\method: A Differentiable Meta Solver for Combinatorial Optimization Problems}

\author{%
  Ruizhong Qiu\thanks{Equal contribution.}\ \,\thanks{Work was done during internship at CMU.} \\
  Department of Computer Science \\
  University of Illinois Urbana--Champaign \\
  \texttt{rq5@illinois.edu} \\
  \And
  Zhiqing Sun\footnotemark[1]\,, Yiming Yang \\
  Language Technologies Institute \\
  Carnegie Mellon University \\
  \texttt{\{zhiqings,yiming\}@cs.cmu.edu} \\
}

\begin{document}

\maketitle

\begin{abstract}
  Recently, deep reinforcement learning (DRL) models have shown promising results in solving NP-hard Combinatorial Optimization (CO) problems.
  However, most DRL solvers can only scale to a few hundreds of nodes for combinatorial optimization problems on graphs, such as the Traveling Salesman Problem (TSP).
  This paper addresses the scalability challenge in large-scale combinatorial optimization by proposing a novel approach, namely, \method.
  Unlike previous DRL methods which suffer from costly autoregressive decoding or iterative refinements of discrete solutions, \method{} introduces a compact continuous space for parameterizing the underlying distribution of candidate solutions.
  Such a continuous space allows stable REINFORCE-based training and fine-tuning via massively parallel sampling.
  We further propose a meta-learning framework to 
  enable effective initialization of model parameters in the fine-tuning stage.
  Extensive experiments show that \method{} outperforms recent DRL-based methods on large benchmark datasets for Traveling Salesman Problems and Maximal Independent Set problems.
\end{abstract}

\section{Introduction}

Combinatorial Optimization (CO) is a fundamental problem in computer science. It has important real-world applications 
such as shipment planning, transportation, robots routing, biology, circuit design, and more
\citep{vesselinova2020learning}. However, due to NP-hardness, 
a significant portion of the CO problems 
suffer from 
an exponential computational cost when using traditional algorithms. As a well-known example, the Traveling Salesman Problem (TSP) has been intensively studied \citep{karp1972reducibility,papadimitriou1977euclidean} for finding
the most cost-effective tour over an input graph where each node is visited exactly once 
before finally returning to the start node. Over the past decades, 
significant effort has been made for designing more efficient heuristic solvers \citep{arora1996polynomial,gonzalez2007handbook} 
to approximate near-optimal solutions in a reduced search space.

Recent development in deep reinforcement learning (DRL) has shown promises in 
solving CO problems without manual injection of domain-specific expert knowledge \citep{bello2016neural,19iclr-am,kwon2020pomo}. The appeal of neural methods is because they can learn useful patterns (such as graph motifs) from data, which might be difficult to discover by hand.  
A typical category of DRL solvers, namely \textit{construction heuristics learners}, \citep{bello2016neural,19iclr-am} 
uses a Markov decision process (MDP) to grow partial solutions by adding one new node per step, with a trained strategy which assigns higher probabilities to better solutions.
Another category of DRL-based solvers, namely \textit{improvement heuristics learners} \citep{chen2019learning,wu2021learning}, 
iteratively refines a feasible solution with neural network-guided local Operations Research (OR) operations 
\citep{reeves1997modern}.
A major limitation of these DRL solvers lies in their scalability on large instances.  For example, current DRL solvers for TSP can only scale to graphs with up to hundreds of nodes.

The bad scalability of these DRL methods lies in the fact that they suffer from costly decoding of CO solutions, which is typically linear in the number of nodes in the input graph. 
Since the reward of reinforcement learning is determined after decoding a complete solution (with either a chain rule factorization or iterative refinements), either construction or improvement heuristic learners would encounter the sparse reward problem when dealing with large graphs \citep{19iclr-am,joshi2020learning,kim2021learning}.
While such an overhead can be partially alleviated by constructing several parts of the solution in parallel \citep{ahn2020learning} for locally decomposable CO problems\footnote{Locally decomposable problem refers to the problem where the feasibility constraint and the objective can be decomposed by locally connected variables (in a graph) \citep{ahn2020learning}.}, such as for maximum independent set (MIS) problems \citep{moon1965cliques}, 
how to scale up neural solvers for CO problems in general, including the locally non-decomposable ones (such as TSP)
is still an open challenge. 

In this paper, we address the scalability challenge by proposing a novel framework, namely \method{} (DIfferentiable MEta Solver), for solving combinatorial optimization problems.
Unlike previous DRL-based CO solvers that rely on
construction or improvement heuristics, we introduce a compact continuous space to parameterize the underlying distribution of candidate solutions, 
which allows massively parallel on-policy sampling without the costly decoding process, and effectively reduces the variance of the gradients by the REINFORCE algorithm \citep{williams1992simple} during both training and fine-tuning phases.
We further propose a meta-learning framework for CO over problem instances to enable effective initialization of model parameters in the fine-tuning stage.
To our knowledge, we are the first to apply meta-learning over a collection of CO problem instances, where each instance graph is treated as one of a collection tasks 
in a unified framework.

We need to point out that the idea of designing a continuous space for combinatorial optimization problems has been tried by the heatmaps approaches in the literature \citep{li2018combinatorial,joshi2019efficient,fu2020generalize,drori2020learning,kool2021deep}. 
However, there are major distinctions between the existing methods and our \method. 
For instance, \citet{fu2020generalize} learn to generate heatmaps via supervised learning (i.e., each training instance is paired with its best solution) \citep{applegate2006concorde,gurobi2018gurobi}, which is very costly to obtain on large graphs.
\method{} is directly optimized with gradients estimated by the REINFORCE algorithm without any supervision, so it can be trained on large graphs directly.
As a result, \method{} can scale to large graphs with up to tens of thousands of nodes, and predict (nearly) optimal solutions without the need for costly generation of supervised training data or human specification of problem-specific heuristics. 

In our experiments, we show that \method{} outperforms strong baselines among DRL-based solvers on TSP benchmark datasets, and can successfully scale up to graphs with tens of thousands of nodes. As a sanity check, we also evaluate our framework with locally decomposable combinatorial optimization problems, including Maximal Independent Set (MIS) problem for synthetic graphs and graphs reduced from satisfiability (SAT) problems. Our experimental results show that \method{} achieve competitive performance compared to neural solvers specially designed for locally decomposable CO problems.

\section{Related Work}

\subsection{DRL-Based Construction Heuristics Learners}

Construction heuristics methods create a solution of CO problem instance in one shot without further modifications.
\citet{bello2016neural} are the first to tackle combinatorial optimization problems
using neural networks and reinforcement learning. They used a Pointer Network (PtrNet) \citep{pointer-network} as the policy network and used the actor-critic algorithm \citep{konda2000actor} for training on TSP and KnapSack instances.
Further improved models have been developed afterwards
\citep{ean, 19iclr-am, peng2019deep, drori2020learning, kwon2021matrix}, such as attention models \citep{vaswani2017attention}, better DRL algorithms \citep{s2v-dqn,ma2019combinatorial,kool2019buy,kwon2020pomo,ouyang2021improving,xin2021multi,wang2021game},
for an extended scope of CO problems 
such as Capacitated Vehicle Routing Problem (CVRP) \citep{nazari2018reinforcement}, Job Shop Scheduling Problem (JSSP) \citep{zhang2020learning}, Maximal Independent Set (MIS) problem \citep{s2v-dqn,ahn2020learning}, and boolean satisfiability problem (SAT) \citep{yolcu2019learning}.

Our proposed method in this paper belongs to the category of \emph{construction heuristics learners} 
in the sense of producing a one-shot solution per problem instance.
However, there are major distinctions between previous methods and ours. One distinction is how to construct solutions. Unlike previous methods which generate the solutions via a constructive Markov decision process (MDP) with rather costly decoding steps (adding one un-visited node per step to a partial solution), we introduce a compact continuous space to parameterize the underlying distribution of discrete candidate solutions, and to allow efficient sampling from that distribution without costly neural network-involved decoding. Another distinction is about the training framework. For instance, \citet{drori2020learning} proposes a similar solution decoding scheme but employs a DRL framework to train the model. Instead, we propose a much more effective meta-learning framework to train our model, enabling \method{} to be trained on large graphs directly.


\subsection{DRL-Based Improvement Heuristics Learners }

In contrast to construction heuristics, DRL-based improvement heuristics methods train a neural network to 
iteratively 
improve the quality of the current solution 
until 
computational budget runs out. Such DRL-based improvement heuristics methods are usually inspired by classical local search algorithms such as 2-opt \citep{croes1958method} and the large neighborhood search (LNS) \citep{shaw1997new}, and have been demonstrated with outstanding results by many previous work \citep{wu2021learning,l2i,wang2021bi,da2020learning,chen2019learning,hottung2019neural,xin2021neurolkh,ma2021learning,hudson2021graph,kim2021learning}.
Improvement heuristics methods generally show better performance than construction heuristics methods but are slower in computation in return.

\subsection{Supervised Learners for CO Problems}

\citet{pointer-network} trained a Pointer Network to predict a TSP solution based on supervision signals from the Held--Karp algorithm \citep{bellman1962dynamic} or approximate algorithms. \citet{li2018combinatorial} and \citet{joshi2019efficient} trained a graph convolutional network to predict the possibility of each node or edge to be included in the optimal solutions of MIS and TSP problems, respectively.
Recently, \citet{joshi2019learning} showed that unsupervised reinforcement learning 
leads to better emergent generalization over various sized graphs 
than supervised learning. Our work in this paper provides further evidence for the benefits of the unsupervised training, or more specifically, unsupervised generation of heatmaps \citep{li2018combinatorial,joshi2019efficient,fu2020generalize,kool2021deep}, for combinatorial optimization problems.

\section{Proposed Method}\label{sec:method}


\subsection{Formal Definitions} \label{sec:31}


Following a conventional notation \citep{papadimitriou1998combinatorial} we define $\mathcal{F}_s$ as the set of discrete feasible solutions for a CO problem instance $s$, and $c_s: \mathcal{F}_s \rightarrow \mathbb{R}$ as the cost function for feasible solutions $f \in \mathcal{F}_s$. The objective is to find the optimal solution for a given instance $s$:
\begin{equation}
    f^*_s=\argmin_{f\in\CAL F_s} c_s(f).
\end{equation}
For the Traveling Salesman Problem (TSP), 
$\CAL F_s$ is the set of all the tours 
that visit each node exactly once and returns to the starting node at the end, 
and $c_s$ calculates the cost for each tour $f \in \CAL F_s$ by summing up the edge weights in the tour. The size of $\CAL F_s$ for TSP is $n!$ for a graph with $n$ nodes.
For the Maximal Independent Set (MIS) problem, $\CAL F_s$ is a subset of the power set $\CAL S_s = \{0, 1\}^n$ and consists of
all the independent subsets where each node of a subset has no connection to any other node in the same subset,
and $c_s$ calculates the negation of the size of each independent subset.

We parameterize the solution space with a continuous and differentiable vector $\boldsymbol{\theta} \in \mathbb{R}^{|\CAL V_s|}$, where $\CAL V_s$ denotes the variables in the problem instance $s$ (e.g., edges in TSP and nodes in MIS), and estimates the probability of each feasible solution $f$ as:
\begin{equation} \label{eq:policy}
    p_{\boldsymbol{\theta}}(f \mid s) \propto  \exp\bigg(\sum_{i=1}^{|\CAL V_s|} f_i \cdot \theta_i\bigg) \text{\quad subject to\quad} f \in \CAL F_s.
\end{equation}
where 
$p_{\boldsymbol{\theta}}$ is  
an energy function over the discrete feasible solution space, $f$ is a $|\CAL V_s|$-dimensional vector with element $f_i \in \{0, 1\}$ indicating whether the $i^\text{th}$ variable is included in feasible solution $f$, and the higher value of $\theta_i$ means a higher probability for the $i^\text{th}$ variable produced by $p_{\boldsymbol{\theta}}(f\mid s) $.




\subsection{Gradient-Based Optimization} \label{sec:32}

When the combinatorial problem is locally decomposable, such a MIS, a penalty loss \citep{karalias2020erdos,alkhouri2022differentiable} can be added to suppress the unfeasible solutions, e.g.:
\begin{equation}
    \ell_{\text{Erd\H os}}(\boldsymbol{\theta} \mid s) = \sum_{f \in \CAL S_s}\left[p_{\boldsymbol{\theta}}(f\mid s) \cdot (c_s(f) + \beta \cdot \mathbbm{1}(f \not\in \CAL F_s))\right].
\end{equation}
where $\beta > \max_{f \in \CAL F_s} c_s(f)$.
The objective function $\ell_{\text{Erd\H os}}$ can thus be calculated analytically and enable end-to-end training. However, this is not always possible for general structured combinatorial problems such as TSP\footnote{TSP has a global constraint of forming a Hamiltonian cycle.}. Therefore, we propose to directly optimize the expected cost over the underlying population of feasible solutions, which is defined as:
\begin{equation}
    \ell_p(\boldsymbol{\theta}\mid s) = \mathbb{E}_{f \sim p_{\boldsymbol\theta}}\left[c_s(f)\right].
\end{equation}
Optimizing this objective requires efficient sampling, with which 
REINFORCE-based \citep{williams1992simple} gradient estimation can be calculate. Nevertheless, a common practice to sample from
the energy $p_{\boldsymbol\theta}$ functions requires MCMC \citep{lecun2006tutorial}, which is not efficient enough. Hence we propose to design an auxiliary distribution $q_{\boldsymbol{\theta}}$ over the feasible solutions $\CAL F_s$, such that the following conditions hold:
1) sampling from $q_{\boldsymbol{\theta}}$ is efficient, and
2) $q_{\boldsymbol{\theta}}$ and $p_{\boldsymbol{\theta}}$ should convergence to the same optimal $\boldsymbol{\theta}^*$.
Then, we can replace $p_{\boldsymbol\theta}$ by $q_{\boldsymbol{\theta}}$ in our objective function as:
\begin{equation} \label{eq:loss}
    \ell_q(\boldsymbol{\theta}\mid s) = \mathbb{E}_{f \sim q_{\boldsymbol\theta}}\left[c_s(f)\right],
\end{equation}
and get the REINFORCE-based update rule as:
\begin{equation} \label{eq:reinforce}
\nabla_{\boldsymbol{\theta}}\mathbb{E}_{f \sim q_{\boldsymbol\theta}}\left[c_s(f)\right] = \mathbb{E}_{f \sim q_{\boldsymbol{\theta}}}[(c_s(f) - b(s)) \nabla_{\BM\theta}\log q_{\BM\theta}(f)],
\end{equation}
where $b(s)$ denotes a baseline function that does not depend on $f$ and estimates the expected cost to reduce the variance of the gradients. In this paper, we use a sampling-based baseline 
function proposed by \citet{kool2019buy}.

Next, we specify the auxiliary distributions for TSP and MIS, respectively.  For brevity, we omit the conditional notations of $s$ for all probability formulas in the rest of the paper.

\subsubsection{Auxiliary Distribution for TSP}
For TSP on an $n$-node graph, each feasible solution $f$ consists of $n$ edges forming a tour, which can be specified as a permutation $\pi_f$ of $n$ nodes, where $\pi_f(0)=\pi_f(n)$ is the start/end node, and $\pi_f(i) \neq \pi_f(j)$ for any $i,j$ with $0\le i,j<n$ and $i \neq j$. Note that for a single solution $f$, $n$ different choices of the start node $\pi_f(0)$ correspond to $n$ different permutations $\pi_f$. In this paper, we choose the start node $\pi_f(0)$ randomly with a uniform distribution:
\begin{gather}
q_\text{TSP}(\pi_f(0)=j):=\frac1n\quad\text{for any node }j;\\
q^{\text{TSP}}_{\boldsymbol{\theta}}(f):=\sum_{j=0}^{n-1} \frac{1}{n}\cdot q_{\text{TSP}}(\pi_f\mid \pi_f(0)=j).
\end{gather}
Given the start node $\pi_f(0)$, we factorize the probability via chain rule in the visiting order:
\begin{equation}
    q_{\text{TSP}}(\pi_f\mid\pi_f(0)):= \prod_{i=1}^{n-1} q_{\text{TSP}}(\pi_f(i)\mid\pi_f(<i)).
\end{equation}
Since the variables in TSP are edges, we let $\theta_{i,j}$ denote the $\theta$ value of edge from node $i$ to node $j$ for notational simplicity, i.e., we use a matrix $\boldsymbol{\theta} \in \mathbb{R}^{n\times n}$ to parameterize the probabilistic distribution of $n!$ discrete feasible solutions. We define:
\begin{equation}
    q_{\text{TSP}}(\pi_f(i)\mid\pi_f(<i)) :=
    \frac{\exp(\theta_{\pi_f(i-1),\pi_f(i)})}{\sum_{j=i}^n \exp(\theta_{\pi_f(i-1),\pi_f(j)})}.
\end{equation}
Here a higher valued $\theta_{i,j}$ corresponds to a higher probability for the edge from node $i$ to node $j$ to be sampled. 
The compact, continuous and differentiable space of $\boldsymbol\theta$ allows us to leverage gradient-based optimization without costly MDP-based construction of feasible solutions, which has been a bottleneck for scaling up in representative DRL solvers so far.
In other words, we also no longer need costly MCMC-based sampling for optimizing our model due to the chain-rule decomposition.
Instead, we use autoregressive factorization for sampling from the auxiliary distribution, which is faster than sampling with MCMC from the distribution defined by the energy function.


\subsubsection{Auxiliary Distribution for MIS}
For the Maximal Independent Set (MIS) problem, the feasible solution is a set of independent nodes, which means that none of the node has any link to any other node in the same set.
To ease the analysis, we further impose a constraint to the MIS solutions such that each set is not a proper subset of any other independent set in the feasible domain.

To enable the chain-rule decomposition in probability estimation, we introduce $\boldsymbol{a}$ as an ordering of the independent nodes in solution $f$, and $\{\boldsymbol{a}\}_f$ as the set of all possible orderings of the nodes in $f$.
The chain rule applied to $\boldsymbol{a}$ can thus be defined as:
\begin{align}\label{eq:prob-mis}
    q_{\boldsymbol\theta}^{\text{MIS}}(f) &= \sum_{\boldsymbol{a} \in \{\boldsymbol{a}\}_f} q_{\text{MIS}}(\boldsymbol{a}),\\
    q_{\text{MIS}}(\boldsymbol{a})
    &=\prod_{i=1}^{|\boldsymbol{a}|} q_{\text{MIS}}(a_i\ |\ \boldsymbol{a}_{<i})\nonumber
    = \prod_{i=1}^{|\boldsymbol{a}|} \frac{\exp(\theta_{a_i})}{\sum_{j \in \mathcal{G}(\boldsymbol{a}_{<i})} \exp(\theta_j)}.
\end{align}
where $\mathcal{G}(\boldsymbol{a}_{<i})$ denotes the set of available nodes
for growing partial solution $(a_1, \dots, a_{i-1})$, i.e., the nodes that have no edge to any nodes in $\{a_1, \dots, a_{i-1}\}$.
Notice again that the parameterization space for MIS $\boldsymbol{\theta} \in \mathbb{R}^{n}$ (where $n$ denotes the number of nodes in the graph) is compact, continuous and differentiable, which allows efficient gradient-driven optimization.

Due to the space limit, we leave the proof of the convergence between $p_{\boldsymbol{\theta}}$ and  $q_{\boldsymbol{\theta}}$ (i.e., $q^{\text{TSP}}_{\boldsymbol{\theta}}$ and $q^{\text{MIS}}_{\boldsymbol{\theta}}$) to the appendix.


\subsection{Meta-Learning Framework} \label{sec:33}

Model-Agnostic Meta-Learning (MAML) \citep{finn2017model} is originally proposed for few-shot learning. In the MAML framework, a model is first trained on a collection of \emph{tasks} simultaneously, and then adapts its model parameters to each task. The standard MAML uses second-order derivatives in training, which are costly to compute. To reduce computation burden, the authors also propose first-order approximation that does not require second-order derivatives.

Inspired by MAML, we train a graph neural network (GNN) over a collection of \emph{problem instances} in a way that the it can capture the common nature across all the instances, and adapt its distribution parameters effectively to each instance based on the features/structure of each input graph.
Let $F_{\boldsymbol{\varPhi}}$ be the graph neural network with parameter $\boldsymbol{\varPhi}$, and denote by  $\boldsymbol{\kappa}_s$ the input features of an instance graph $s$ in collection $\mathcal{C}$, by $\boldsymbol A_s$ the adjacency matrix of the input graph, and by 
$\boldsymbol{\theta}_s:=F_{\boldsymbol{\varPhi}}(\boldsymbol{\kappa}_s,\boldsymbol A_s)$ the instance-specific initialization of distribution parameters.
The vanilla loss function is defined as
the expected cost of the solution for any graph in the collection as:
\begin{equation}\label{eq:standard_rl}
\begin{split}
 \mathcal{L}(\boldsymbol{\varPhi} \mid \mathcal{C})
&= \mathbb{E}_{s \in \mathcal{C}}
\ell_q(\boldsymbol{\theta}_s)
=\mathbb{E}_{s \in \mathcal{C}}
\ell_q(F_{\boldsymbol{\varPhi}}(\boldsymbol{\kappa}_s,\boldsymbol A_s)).
\end{split}
\end{equation}
%
The gradient-based updates can thus be written as:
\begin{equation} \label{eq:normal_grad}
\begin{split}
\nabla_{\boldsymbol{\varPhi}} \mathcal{L}(\boldsymbol{\varPhi} \mid \mathcal{C})
&= \mathbb{E}_{s \in \mathcal{C}} \left[\nabla_{\boldsymbol{\varPhi}} {\boldsymbol{\theta}_s}\cdot  \nabla_{\boldsymbol{\theta}_s} \ell_q(\boldsymbol{\theta}_s)\right]\\
&= \mathbb{E}_{s \in \mathcal{C}} \left[\nabla_{\boldsymbol{\varPhi}} {F_{\boldsymbol{\varPhi}}(\boldsymbol{\kappa}_s,\boldsymbol A_s)}\cdot  \nabla_{\boldsymbol{\theta}_s} \ell_q(\boldsymbol{\theta}_s)\right]. 
\end{split}
\end{equation}
where $\nabla_{\boldsymbol{\theta}_s} \ell_q(\boldsymbol{\theta}_s)$ 
is estimated using the REINFORCE algorithm (Equation~\ref{eq:reinforce}).
Since $\ell_q$ does not depend on the ground-truth labels, we can further fine-tune neural network parameters on each single test instance with REINFORCE-based updates, which is referred to as \emph{active search} \citep{bello2016neural,hottung2021efficient}. 

Specifically, the fine-tuned parameters $\boldsymbol{\varPhi}_s^{(T)}$
is computed using one or more gradient updates for each graph instance $s$.
For example, when adapting to a problem instance $s$ using $T$ gradient updates with learning rate $\alpha$, we have:
\begin{gather}\label{eq:test-gradient}
    \boldsymbol{\varPhi}_s^{(0)} = \boldsymbol{\varPhi}, \quad \quad
    \boldsymbol{\varPhi}_s^{(t)} =  \boldsymbol{\varPhi}_s^{(t-1)} - \alpha \nabla_{\boldsymbol{\varPhi}_s^{(t-1)}} \mathcal{L}(\boldsymbol{\varPhi}_s^{(t-1)} \mid \{s\}) \text{\quad for\quad} 1 \le t \le T,\\
    \boldsymbol{\theta}_s^{(T)} = F_{{\boldsymbol\varPhi}^{(T)}_s}(\boldsymbol{\kappa}_s,\boldsymbol A_s).
\end{gather}
Here we use AdamW \cite{adamw} in our experiments. Next, we optimize the performance of the graph neural network with updated parameters (i.e., $\boldsymbol{\varPhi}^{(T)}_s$) with respect to $\boldsymbol{\varPhi}$, with a meta-objective:
\begin{equation} \label{eq:meta_rl}
\begin{split}
 \mathcal{L}_{\text{meta}}(\boldsymbol{\varPhi} \mid \mathcal{C})
&= \mathbb{E}_{s \in \mathcal{C}}
\ell_q(\boldsymbol\theta_s^{(T)}\mid s),
\end{split}
\end{equation}
and calculate the meta-updates as:
\begin{equation} \label{eq:maml_grad}
\begin{split}
\nabla_{\boldsymbol{\varPhi}} \mathcal{L}_{\text{meta}}(\boldsymbol{\varPhi} \mid \mathcal{C})
&= \mathbb{E}_{s \in \mathcal{C}} \left[
\nabla_{\boldsymbol{\varPhi}}
{\boldsymbol{\theta}^{(T)}_s} \cdot  \nabla_{\boldsymbol{\theta}^{(T)}_s} \ell_q(\boldsymbol{\theta}^{(T)}_s)\right] \\
&\approx \mathbb{E}_{s \in \mathcal{C}} \left[
\nabla_{\boldsymbol{\varPhi}^{(T)}_s}{F_{\boldsymbol{\varPhi}^{(T)}_s}(\boldsymbol{\kappa}_s,\boldsymbol A_s)} \cdot  \nabla_{\boldsymbol{\theta}^{(T)}_s} \ell_q(\boldsymbol{\theta}^{(T)}_s)\right].
\end{split}
\end{equation}
Notice that we adopt the first-order approximation to optimize this objective, which ignores the 
update via the gradient term of $\nabla_{\boldsymbol{\varPhi}}\mathcal{L}(\boldsymbol{\varPhi} \mid \{s\})$. We defer the derivation of the approximation formula to the appendix. Algorithm~\ref{alg:mamlrl} illustrates the full training process of our meta-learning framework.

\renewcommand\algorithmicrequire{\textbf{Input:}}
\renewcommand\algorithmicensure{\textbf{Output:}}
\begin{algorithm}[t]
\caption{MAML in \method{}}
\label{alg:mamlrl}
\begin{algorithmic}[1]
{\footnotesize
\REQUIRE $p(\mathcal{C})$: distribution over CO problem instances
\REQUIRE $\alpha$: step size hyperparameters
\STATE randomly initialize $\boldsymbol{\varPhi}$
\WHILE{not done}
\STATE Sample batch of graph instances $c_i \sim p(\mathcal{C})$
  \FORALL{$c_i$}
      \STATE Sample $K$ solutions $\mathcal{D}_i = \{f_1, f_2, \dots, f_K\}$ using $q_{F_{\boldsymbol{\varPhi}}(\boldsymbol{\kappa}_s,\boldsymbol A_s)}$ for $c_i$
      \STATE Evaluate $\nabla_{\boldsymbol{\varPhi}} \ell_q(F_{\boldsymbol{\varPhi}}(\boldsymbol{\kappa}_s,\boldsymbol A_s))$ using $\mathcal{D}$ in Equation~\ref{eq:normal_grad}
      \STATE Compute adapted parameters with Equation~\ref{eq:test-gradient}: $\boldsymbol{\varPhi}^{(T)}_i= \mathrm{GradDescent}^{(T)}(\boldsymbol{\varPhi})$
      \STATE  Sample $K$ solutions $\mathcal{D}'_i = \{f'_1, f'_2, \dots, f'_K\}$ using $q_{F_{\boldsymbol{\varPhi}^{(T)}_s}(\boldsymbol{\kappa}_s,\boldsymbol A_s)}$ for $c_i$
 \ENDFOR
 \STATE Update $\boldsymbol{\varPhi}= \boldsymbol{\varPhi} - \mathrm{AdamW}\big(\sum_{c_i \in p(\mathcal{C})} \nabla_{\boldsymbol{\varPhi}} \ell_q(F_{\boldsymbol{\varPhi}^{(T)}_s}(\boldsymbol{\kappa}_s,\boldsymbol A_s))\big)$ using each $\mathcal{D}_i'$ in Equation~\eqref{eq:maml_grad}
\ENDWHILE
}
\end{algorithmic}
\end{algorithm}

\subsection{Per-Instance Search} \label{sec:34}
Given a fine-tuned (i.e., after active search) continuous parameterization of the solution space $\boldsymbol{\theta}_s^{(T)}$, the per-instance search decoding aims to search for a feasible solution that minimizes the cost function $c$. In this paper, we adotp three decoding strategies, i.e., greedy decoding, sampling, and Monte Carlo tree search. Due to the space limit, the detailed description of three decoding strategies can be found in the appendix.

\subsection{Graph Neural Networks}

Based on the shape of the differentiable variable $\boldsymbol{\theta}$ required by each problem (i.e., $\mathbb{R}^{n\times n}$ for TSP and $\mathbb{R}^{n}$ for MIS), we use Anisotropic Graph Neural Networks \citep{bresson2018experimental} and Graph Convolutional Networks \citep{kipf2016semi} as the backbone network for TSP and MIS tasks, respectively. Due to the space limit, the detailed neural architecture design can be found in the appendix.

\section{Experiments}

\subsection{Experiments for Traveling Salesman Problem}

\subsubsection{Experimental Settings} \label{sec:exp}


\paragraph{Data Sets} The training instances are generated on the fly. We closely follow the data generation procedure of previous works, e.g., \citep{19iclr-am}. We generate 2-D Euclidean TSP instances by sampling each node independently from a uniform distribution over the unit square.
The TSP problems of different scales are named TSP-500/1000/10000, respectively, where TSP-$n$ indicates the TSP instance on $n$ nodes.
For testing, we use the test instances generated by \citet{fu2020generalize}. There are 128 test instances in each of TSP-500/1000, and 16 test instances in TSP-10000.

\paragraph{Evaluation Metrics}
For model comparison, we report the average length (Length), average performance drop (Drop) and averaged inference latency time (Time), respectively, where \textit{Length} (the shorter, the better) is the average length of the system-predicted tour for each test-set graph, \textit{Drop} (the smaller, the better) is the average of relative performance drop in terms of the solution length compared to a baseline method, and \textit{Time} (the smaller, the better) is the total clock time for generating solutions for all test instance, in seconds (s), minutes (m), or hours (h).

\paragraph{Training and Hardware} Due to the space limit, please refer to the appendix.

\setlength{\tabcolsep}{4pt}
\begin{table}[t]\scriptsize
\caption{Results of TSP. See Section~\ref{sec:main} for detailed descriptions. * indicates the baseline for computing the performance drop. 
Results of baselines (except those of EAS and the running time of LKH-3, POMO, and Att-GCN) are taken from \citet{fu2020generalize}.}
\label{tab:exp-500-1k-10k}
\begin{center}
\begin{tabular}{ll|ccc|ccc|ccc}
\toprule
\multirow{2}*{Method}&\multirow{2}*{Type}
&\multicolumn{3}{c|}{TSP-500}&\multicolumn{3}{c|}{TSP-1000}&\multicolumn{3}{c}{TSP-10000}\\
&&Length $\downarrow$&Drop $\downarrow$&Time $\downarrow$&Length $\downarrow$&Drop $\downarrow$&Time $\downarrow$&Length $\downarrow$&Drop $\downarrow$&Time $\downarrow$\\
\midrule
Concorde&OR (exact)
&16.55$^*$&---&37.66m&23.12$^*$&---&6.65h                   &N/A&N/A&N/A\\
Gurobi&OR (exact) 
&16.55&0.00\%&45.63h&N/A&N/A&N/A                 &N/A&N/A&N/A\\
LKH-3 (default)&OR
&16.55&0.00\%&46.28m
&23.12&0.00\%&2.57h
&71.77$^*$&---&8.8h
\\
LKH-3 (less trails) &OR&16.55&0.00\%&3.03m&23.12&0.00\%&7.73m &71.79&---&51.27m
\\
Nearest Insertion &OR & 20.62 & 24.59\% & 0s& 28.96 & 25.26\% & 0s & 90.51 & 26.11\% & 6s\\
Random Insertion &OR & 18.57 & 12.21\% & 0s& 26.12 & 12.98\% & 0s & 81.85 & 14.04\% & 4s\\
Farthest Insertion &OR & 18.30 & 10.57\% & 0s& 25.72 & 11.25\% & 0s & 80.59 & 12.29\% & 6s\\
\midrule
EAN&RL+S
&28.63&73.03\%&20.18m&50.30&117.59\%&37.07m     &N/A&N/A&N/A\\
EAN&RL+S+2-OPT
&23.75&43.57\%&57.76m&47.73&106.46\%&5.39h      &N/A&N/A&N/A\\
AM&RL+S
&22.64&36.84\%&15.64m&42.80&85.15\%&63.97m      &431.58&501.27\%&12.63m\\
AM&RL+G
&20.02&20.99\%&1.51m&31.15&34.75\%&3.18m        &141.68&97.39\%&5.99m\\
AM&RL+BS
&19.53&18.03\%&21.99m&29.90&29.23\%&1.64h       &129.40&80.28\%&1.81h\\
GCN&SL+G
&29.72&79.61\%&6.67m&48.62&110.29\%&28.52m      &N/A&N/A&N/A\\
GCN&SL+BS
&30.37&83.55\%&38.02m&51.26&121.73\%&51.67m     &N/A&N/A&N/A\\

POMO+EAS-Emb&RL+AS&19.24&16.25\%&12.80h&N/A&N/A&N/A&N/A&N/A&N/A\\
POMO+EAS-Lay&RL+AS&19.35&16.92\%&16.19h&N/A&N/A&N/A&N/A&N/A&N/A\\
POMO+EAS-Tab&RL+AS&24.54&48.22\%&11.61h&49.56&114.36\%&63.45h&N/A&N/A&N/A\\
Att-GCN&SL+MCTS & 16.97 & 2.54\% & 2.20m & 23.86 & 3.22\% & 4.10m & 74.93 & 4.39\% & 21.49m \\
\midrule
\multirow{6}*{\method{} (ours)}&RL+G
&18.93 &14.38\% &0.97m & 26.58 & 14.97\% & 2.08m & 86.44 & 20.44\% & 4.65m\\%
&RL+AS+G
&17.81&7.61\%&2.10h&24.91&7.74\%&4.49h        &80.45&12.09\%&3.07h\\
&RL+S
&18.84 &13.84\% &1.06m & 26.36 & 14.01\% & 2.38m & 85.75 &19.48\% & 4.80m\\%
&RL+AS+S &17.80&7.55\%&2.11h&24.89&7.70\%&4.53h&80.42&12.05\%&3.12h\\
&RL+MCTS
& 16.87 & 1.93\% & 2.92m & 23.73 & 2.64\% & 6.87m & 74.63 & 3.98\% & 29.83m\\%
&RL+AS+MCTS
&\textbf{16.84} &\textbf{1.76\%} &2.15h
&\textbf{23.69} &\textbf{2.46\%} &4.62h
&\textbf{74.06} &\textbf{3.19\%} &3.57h
\\
\bottomrule
\end{tabular}
\end{center}
\vspace*{-2mm}
\end{table}
\setlength{\tabcolsep}{6pt}

\begin{table}[t]\begin{center}\scriptsize
\caption{Ablation study on TSP-1000.}
\begin{subtable}{0.58\linewidth}
\begin{center}
\caption{On meta-learning ($T=10$).}\label{tab:abla-meta}
\begin{tabular}{ccc}
\toprule
Inner updates & Fine-tuning & Length $\downarrow$\\
\midrule
  & &27.11\\
\checkmark&  &26.58\\
& \checkmark &25.68\\
\checkmark & \checkmark &\textbf{24.91}\\
\bottomrule
\end{tabular}
\end{center}
\end{subtable}
\begin{subtable}{0.38\linewidth}
\begin{center}
\caption{On fine-tuning parts ($T=5$).}\label{tab:abla-tune}
\begin{tabular}{lc}
\toprule
Part & Length $\downarrow$\\
\midrule
Cont. Param. & 27.73\\
MLP & 26.75\\
GNNOut+MLP & \textbf{26.49}\\
GNN+MLP & 26.81\\
\bottomrule
\end{tabular}
\end{center}
\end{subtable}\\
\begin{subtable}{0.58\linewidth}
\begin{center}
\caption{On inner update steps $T$.}\label{tab:abla-meta-upd}
\begin{tabular}{ccccccc}
\toprule
$T$&0&4&8&10&12&14\\
\midrule
Length $\downarrow$&25.79&25.28&25.08&25.08&24.97&24.91\\
\bottomrule
\end{tabular}
\end{center}
\end{subtable}
\begin{subtable}{0.38\linewidth}
\begin{center}
\caption{On heatmaps for MCTS.}\label{tab:abla-mcts}
\begin{tabular}{lc}
\toprule
Heatmap & Length $\downarrow$\\
\midrule
$\operatorname{Unif}(0,1)$&25.52\\
$1/(r_i+1)$&24.14\\
\midrule
Att-GCN&23.86\\
\method{} (ours)&\textbf{23.69}\\
\bottomrule
\end{tabular}
\end{center}
\end{subtable}
\end{center}
\vspace*{-4mm}
\end{table}

\subsubsection{Main Results} \label{sec:main}

Our main results are summarized in Table~\ref{tab:exp-500-1k-10k}, with $T=15$ for TSP-500, $T=14$ for TSP-1000, and $T=12$ for TSP-10000.
We use a GNN followed by an MLP as the backbone, whose detailed architecture is defered to the appendix. Note that we fine-tune the GNN output and the MLP only.
For the evaluation of \method, we fine-tune the \method{} on each instance for 100 steps (TSP-500 \& TSP-1000) or for 50 steps (TSP-10000).
For the sampling in \method, we use the temperature parameter $\tau=0.01$ for \method+S and $\tau=1$ for \method+AS+S.
We compare \method{} with 14 other TSP solvers on the same test sets. We divide those 14 methods into two categories: 6 traditional OR methods and 8 learning-based methods.
\begin{itemize}[leftmargin=0.15in]
\item Traditional operations research methods include two exact solvers, i.e., Concorde \citep{applegate2006concorde} and Gurobi \citep{gurobi2018gurobi}, and a strong heuristic solver named LKH-3 \citep{lkh3}. For LKH-3, we consider two settings: (i) \emph{default}: following previous work \cite{19iclr-am}, we perform 1 runs with a maximum of 10000 trials (the default configuration of LKH-3); (ii) \emph{less trials}: we perform 1 run with a maximum of 500 trials for TSP-500/1000 and 250 trials for TSP-10000, so that the running times of LKH-3 match those of \method+MCTS. Besides, we also compare \method{} against simple heuristics, including Nearest, Random, and Farthest Insertion.
\item  Learning-based methods include 8 variants of the 4 methods with the strongest results in recent benchmark evaluations, namely EAN \citep{ean}, AM \citep{19iclr-am}, GCN \citep{joshi2019efficient}, POMO+EAS \citep{hottung2021efficient}, and Att-GCN \citep{fu2020generalize}, respectively.  Those methods can be further divided into the reinforcement learning (RL) sub-category and the supervised learning (SL) sub-category.
Some reinforcement learning methods can further adopt an Active Search (AS) stage to fine-tune on each instance.
The results of the baselines except the running time of Att-GCN are taken from \citet{fu2020generalize}. Note that baselines are trained on small graphs and evaluated on large graphs, while \method{} can be trained directly on large graphs. We re-run the publicly available code of Att-GCN on our hardware to ensure fair comparison of time.
\end{itemize}

The decoding schemes in each method (if applicable) are further specified as Greedy decoding (G), Sampling (S), Beam Search (BS), and Monte Carlo Tree Search (MCTS). The 2-OPT improvements \citep{croes1958method} can be optionally used to further improve the neural network-generated solution via heuristic local search. See Section~\ref{sec:34} for a more detailed descriptions of the various decoding techniques.

As is shown in the table, \method{} significantly outperforms many previous learning-based methods. Notably, although \method{} is trained without any ground truth solutions, it is able to outperform the supervised method. \method{} also consistently outperforms simple traditional heuristics. The best performance is achieved by RL+AS+MCTS, which requires considerably more time. RL+AS+G/S are faster than RL+AS+MCTS and are competitive to the simple heuristics. Removing AS in \method{} shortens the running time and leads to only a slight, acceptable performance drop. Moreover, they are still better than many previous learning-based methods in terms of solution quality and inference time.

\setlength{\tabcolsep}{4pt}
\begin{table}[t]\scriptsize
\caption{Results of various methods on MIS problems. Notice that we disable graph reduction and $2$-opt local search in all models for a fair comparison, since it is pointed out by \citep{other2022whats} that all models would perform similarly with a local search post-processing. See Section~\ref{sec:main2} for detailed descriptions. * indicates the baseline for computing the performance drop.}
\label{tab:exp-mis}
\begin{center}
\begin{tabular}{ll|ccc|ccc|ccc}
\toprule
\multirow{2}*{Method}&\multirow{2}*{Type}
&\multicolumn{3}{c|}{SATLIB}&\multicolumn{3}{c|}{ER-[700-800]}&\multicolumn{3}{c}{ER-[9000-11000]}\\
&&Size $\uparrow$&Drop $\downarrow$&Time $\downarrow$&Size $\uparrow$&Drop $\downarrow$&Time $\downarrow$&Size $\uparrow$&Drop $\downarrow$&Time $\downarrow$\\
\midrule
KaMIS&OR
&425.96$^*$&---&37.58m&44.87$^*$&---&52.13m        &381.31$^*$&---&7.6h\\
Gurobi&OR &
425.95& 0.00\%& 26.00m &  41.38 & 7.78\% & 50.00m & N/A&N/A&N/A\\
\midrule
Intel & SL+TS & N/A & N/A & N/A & 38.80 & 13.43\% & 20.00m & N/A & N/A & N/A \\
Intel & SL+G & 420.66 & 1.48\% & 23.05m & 34.86 & 22.31\% & 6.06m & 284.63 & 25.35\% & 5.02m \\
DGL   & SL+TS & N/A & N/A & N/A & 37.26 & 16.96\% & 22.71m & N/A & N/A & N/A \\
LwD & RL+S & 422.22 & 0.88\% & 18.83m & 41.17 & 8.25\% & 6.33m & \textbf{345.88} & \textbf{9.29\%} & 7.56m\\
\midrule
\method{} (ours) &RL+G& 421.24 & 1.11\% & 24.17m & 38.24 & 14.78\% & 6.12m & 320.50 & 15.95\% & 5.21m\\
\method{} (ours) &RL+S& \textbf{423.28} & \textbf{0.63\%} & 20.26m & \textbf{42.06} & \textbf{6.26\%} & 12.01m & 332.80 & 12.72\% & 12.51m\\
\bottomrule
\end{tabular}
\end{center}
\vspace*{-4mm}
\end{table}
\setlength{\tabcolsep}{6pt}

\subsubsection{Ablation Study}

\paragraph{On Meta-Learning}
To study the efficacy of meta-learning, we consider two dimensions of ablations: (i) with or without inner gradient updates: whether to use $f_{\boldsymbol\varPhi}(\boldsymbol \kappa_s,\boldsymbol A_s)$ or $\boldsymbol\theta_s^{(T)}$ in the objective function; (ii) with or without fine-tuning in the inference phase. The results on TSP-1000 with training phase $T=10$ and greedy decoding are summarized in Table~\ref{tab:abla-meta}. Both inner updates and fine-tuning are crucial to the performance of our method. That is because meta-learning helps the model generalize across problem instances, and fine-tuning helps the trained model adapt to each specific problem instance.

\paragraph{On Fine-Tuning Parts}
We study the effect of fine-tuning parts during both training and testing. In general, the neural architecture we used is a GNN appended with an MLP, whose output is the continuous parameterization $\boldsymbol\theta$. We consider the following fine-tuning parts: (i) the continuous parameterization (Cont. Param.); (ii) the parameter of MLP; (iii) the output of GNN and the parameters of MLP; (iv) the parameters of GNN and MLP. Table~\ref{tab:abla-tune} summarizes the results with various fine-tuning parts for TSP-1000 with training phase $T=5$ and greedy decoding. The result demonstrates that (iii) works best. We conjecture that (iii) makes a nice trade-off between universality and variance reduction.

\paragraph{On Inner Gradient Update Steps}
We also study the effect of the number $T$ of inner gradient update steps during training. Table~\ref{tab:abla-meta-upd} shows the test performance on TSP-1000 by greedy decoding with various $T$'s. As the number of inner gradient updates increases, the test performance improves accordingly. Meanwhile, more inner gradient update steps consumes more training time. Hence, there is a trade-off between performance and training time in practice.

\paragraph{On Heatmaps for MCTS}
To study where continuous parameterization of \method{} is essential to good performance in MCTS, we replace it with the following heatmaps: (i) each value is independently sampled from $\operatorname{Unif}(0,1)$; (ii) $1/(r_i+1)$, where $r_i \ge 1$ denotes the rank of the length of the $i$-th edge among those edges that share the source node with it. This can be regarded as an approximation to the nearest neighbor heuristics. We also compare with the Att-GCN heatmap \citep{fu2020generalize}. Comparison of continuous parameterizations for TSP-1000 by MCTS is shown in Table~\ref{tab:abla-mcts}. The result confirms that the \method{} continuous parameterization does not simply learn nearest neighbor heuristics, but can identify non-trivial good candidate edges.

\subsection{Experiments For Maximal Independent Set}

\subsubsection{Experimental Settings}

\paragraph{Data Sets}
We mainly focus on two types of graphs that recent work \citep{li2018combinatorial,ahn2020learning,other2022whats} shows struggles against, i.e., Erd{\H{o}}s-R{\'e}nyi (ER) graphs \citep{erdHos1960evolution} and SATLIB \citep{hoos2000satlib}, where the latter is a set of graphs reduced from SAT instances in CNF. The ER graphs of different scales are named ER-[700-800] and ER-[9000-11000], where ER-[$n$-$N$] indicates the graph contains $n$ to $N$ nodes. The pairwise connection probability $p$ is set to $0.15$ and $0.02$ for ER-[700-800] and ER-[9000-11000], respectively. The 4,096 training and 5,00 test ER graphs are randomly generated.
For SATLIB, which consists of 40,000 instances, of which we train on 39,500 and test on 500. Each SAT instance has between 403 to 449 clauses. Since we cannot find the standard train-test splits for both SAT and ER graphs datasets, we randomly split the datasets and re-run all the baseline methods.

\paragraph{Evaluation Metrics}
To compare the solving ability of various methods, we report the average size of the independent set (Size), average performance drop (Drop) and latency time (Time), respectively, where \textit{Size} (the larger, the better) is the average size of the system-predicted maximal independent set for each test-set graph, \textit{Drop} and \textit{Time} are defined similarly as in Section~\ref{sec:exp}.

\paragraph{Training and Hardware} Due to the space limit, please refer to the appendix.

\subsubsection{Main Results} \label{sec:main2}

Our main results are summarized in Table~\ref{tab:exp-mis}, where our method (last line) is compared 6 other MIS solvers on the same test sets, including two traditional OR methods (i.e., Gurobi and KaMIS) and four learning-based methods.
The active search is not used for MIS evaluation since our preliminary experiments only show insignificant improvements.
For Gurobi, we formulate the MIS problem as a integer linear program. For KaMIS, we use the code unmodified from the official repository\footnote{\url{https://github.com/KarlsruheMIS/KaMIS} (MIT License)}.
The four learning-based methods can be divided into the reinforcement learning (RL) category, i.e., S2V-DQN \citep{s2v-dqn} and LwD \citep{ahn2020learning}) and the supervised learning (SL) category, i.e., Intel \citep{li2018combinatorial} and DGL \citep{other2022whats}.

We produced the results for all the learning-based methods by running an integrated implementation\footnote{\url{https://github.com/MaxiBoether/mis-benchmark-framework} (No License)} provided by \citet{other2022whats}. Notice that as pointed out by \citet{other2022whats}, the graph reduction and local $2$-opt search \citep{andrade2012fast} can find near-optimal solutions even starting from a randomly generated solution, so we disable the local search or graph reduction techniques during the evaluation for all learning based methods to reveal their real CO-solving ability. The methods that cannot produce results in the $10\times$ time limit of \method{} are labeled as N/A.

As is shown in Table~\ref{tab:exp-mis}, our \method{} model outperforms previous baseline methods on the medium-scale SATLIB and ER-[700-800] datasets, and significantly outperforms the supervised baseline (i.e., Intel) on the large-scale ER-[9000-11000] setting. This shows that supervised neural CO solvers cannot well solve large-scale CO problem due to the expensive annotation problem and generalization problem. In contrast, reinforcement-learning methods are a better choice for large-scale CO problems. We also find that LwD outperforms \method{} on the large-scale ER-[9000-11000] setting. We believe this is because LwD is specially designed for locally decomposable CO problems such as MIS and thus can use parallel prediction, but \method{} are designed for general CO problems and only uses autoregressive factorization. How to better utilize the fact of local decomposability of MIS-like problems is one of our future work. 

\section{Conclusion \& Discussion}
\label{sec:concl}

Scalability without significantly scarifying the approximation accuracy is a critical challenge in combinatorial optimization. In this work we proposed \method, a differentiable meta solver that is able to solve large-scale combinatorial optimization problems effectively and efficiently, including TSP and MIS. The 
novel parts of \method{} include the compact continuous parameterization and the meta-learning strategy. Notably, although our method is trained without any ground truth solutions, it is able to outperform several supervised methods. In comparison with other strong DRL solvers on TSP and MIS problems, \method{} can scale up to graphs with ten thousand nodes while the others either fail to scale up, or can only produce significantly worse solutions instead in most cases.

Our unified framework is not limited to TSP and MIS. Its generality is based on the assumption that each feasible solution of the CO problem on hand can be represented with 0/1 valued variables (typically corresponding the selection of a subset of nodes or edges), which is fairly mild and generally applicable to many CO problems beyond TSP and MIS (see Karp's 21 NP-complete problems \cite{karp1972reducibility}) with few modifications. The design principle of auxiliary distributions is to design an autoregressive model that can sequentially grow a valid partial solution toward a valid complete solution. This design principle is also proven to be general enough for many problems in neural learning, including CO solvers. There do exist problems beyond this assumption, e.g., Mixed Integer Programming (MIP), where variables can take multiple integer values instead of binary values. Nevertheless, \citet{nair2020mip} showed that this issue can be addressed by reducing each integer value within range $[U]$ to a sequence of $\lceil\log_2U\rceil$ bits and by predicting the bits from the most to the least significant bits. In this way, a multi-valued MIP problem can be reduced to a binary-valued MIP problem with more variables.

One limitation of \method{} is that the continuous parameterization $\boldsymbol{\theta}$ is generated in one-shot without intermediate steps, which could potentially limit the reasoning power of our method, as is shown in the MIS task.
Another limitation is that applying \method{} to a broader ranges of NP-complete problems that variables can take multiple values, such as Mixed Integer Programming (MIP), is non-trivial and needs further understanding of the nature of the problems.


\newpage
\bibliography{main}
\bibliographystyle{plainnat}

\input{appendix}
\end{document}

%% file: appendix.tex
\newpage
\appendix

\section{Additional Related Work}

\subsection{Per-Instance Search}

Once the neural network is trained 
over a collection of problem instances, per-instance fine-tuning can be used to improve the quality of solutions via local search. 
For DRL solvers, \citet{bello2016neural} fine-tuned the policy network on each test graph, which is referred as \emph{active search}.
\citet{hottung2021efficient} proposed 
three active search strategies for efficient updating of parameter subsets during search.
\citet{zheng2020combining} tried a combination of 
traditional reinforcement learning with Lin-Kernighan-Helsgaun (LKH) Algorithm \citep{lin1973effective,helsgaun2000effective}. \citet{hottung2020learning} performed per-instance search in a differentiable continuous space encoded by a conditional variational auto-encoder \citep{kingma2013auto}. 
With a heatmap indicating the promising parts of the search space, discrete solutions can be found via  
beam search \citep{joshi2019efficient}, sampling \citep{19iclr-am}, guided tree-search \citep{li2018combinatorial}, dynamic programming \citep{kool2021deep}, and Monte Carlo Tree Search (MCTS) \citep{fu2020generalize}. 
In this paper, we mainly adopt greedy, sampling, and MCTS as the per-instance search techniques.



\section{Per-instance Search}

In this section, we describe the decoding strategies used in our paper. Given a fine-tuned (i.e., after active search) continuous parameterization $\boldsymbol{\theta}_s^{(T)}$ of the solution space, the per-instance search decoding aims to search for a feasible solution that minimizes the cost function $c_s$.

\paragraph{Greedy Decoding} generates the solution through a sequential decoding process similar to the auxiliary distribution designed for each combinatorial optimization problem, where at each step, the variable $k$ with the highest score $\theta_{k}$ is chosen to extend the partial solution. For TSP, the first node in the permutation is picked at random.

\paragraph{Sampling} Inspired by \citet{19iclr-am}, we propose to parallelly sample multiple solutions according to the auxiliary distribution and report the best one. The continuous parameterization is divided by a temperature parameter $\tau$. The parallel sampling of solutions in \method{} is very efficient due to the fact that it only relies on the final parameterization $\boldsymbol{\theta}_s^{(T)}/\tau$ but not on neural networks.

\paragraph{Monte Carlo Tree Search} Inspired by \citep{fu2020generalize}, for the TSP task, we also leverage a more advanced reinforcement learning-based searching approach, i.e., Monte Carlo tree search (MCTS), to find high-quality solutions.
In MCTS, $k$-opt transformation actions are sampled guided by the continuous parameterization $\boldsymbol{\theta}_s^{(T)}$ to improve the current solutions.
The MCTS iterates over the simulation, selection, and back-propagation steps, until no improving actions exists among the sampling pool. For more details, please refer to \citep{fu2020generalize}.

\section{Implementation Details} \label{sec:35}
\subsection{Neural Architecture for TSP}\label{apd:arch-tsp}
\paragraph{Anisotropic Graph Neural Networks}

We follow \citet{joshi2020learning} on the choice of neural architectures. The backbone of the graph neural network is an anisotropic GNN with an edge gating mechanism \citep{bresson2018experimental}. Let $\BM h_i^\ell$ and $\BM e_{ij}^\ell$ denote the node and edge features at layer $\ell$ associated with node $i$ and edge $ij$, respectively. The features at the next layer is propagated with an anisotropic message passing scheme:
\begin{align}
    \BM h_i^{\ell+1} &= \BM h_i^\ell + \alpha(\mathrm{BN}(\BM U^\ell \BM h_i^\ell+ \mathcal{A}_{j \in \mathcal{N}_i}(\sigma(\BM e_{ij}^\ell) \odot\BM V^\ell\BM  h^\ell_j))),\\
    \BM e_{ij}^{\ell+1} &= \BM e_{ij}^\ell + \alpha(\mathrm{BN}(\BM P^\ell \BM e^\ell_{ij} + \BM Q^\ell \BM h^\ell_i + \BM R^\ell \BM h^\ell_j)).
\end{align}
where $\BM U^\ell,\BM V^\ell,\BM P^\ell,\BM Q^\ell,\BM R^\ell \in \mathbb{R}^{d\times d}$ are the learnable parameters of layer $\ell$, $\alpha$ denotes the activation function (we use $\mathrm{SiLU}$ \citep{silu} in this paper), $\mathrm{BN}$ denotes the Batch Normalization operator \citep{batch-norm}, $\mathcal{A}$ denotes the aggregation function (we use mean pooling in this paper), $\sigma$ is the sigmoid function, $\odot$ is the Hadamard product, and $\mathcal{N}_i$ denotes the outlinks (neighborhood) of node $i$. We use a 12-layer GNN with width 32.

The node and edge features at the first layer $\BM h_i^{0}$ and $\BM e_{ij}^{0}$ are initialized with the absolute position of the nodes and absolute length of the edges, respectively.
After the anisotropic GNN backbone, a Multi-Layer Perceptron (MLP) is appended and generates the final continuous parameterization $\boldsymbol\theta$ for all the edges. We use a 3-layer MLP with width 32.

\paragraph{Graph Sparsification}

As described, we focus on developing a neural TSP solver for graphs with tens of thousands of nodes. Because the number of edges in the graph grows quadratically to the number of nodes, a densely connected graph is intractable for an anisotropic GNN when it is applied to large graphs. Therefore, we use a simple heuristic to sparsify the original graph. Specifically, we prune the outlinks of each node such that it is only connected to $k$ nearest neighbors. The continuous parameterization $\BM\theta$ is also pruned accordingly. As a result, the computation complexity of our method is reduced from $O(n^2)$ to $O(nk)$, where $n$ is the number of nodes in the graph.

\subsection{Neural Architecture for MIS}

\paragraph{Graph Convolutional Networks} We follow \citet{li2018combinatorial} on the choice of neural architecture, i.e., using Graph Convectional Network (GCN) \citep{kipf2016semi}, since $\BM\theta$ is merely scores for each node. Specifically, the GCN backbone consists of multiple layers $\{\mathbf{h}^l\}$ where $\mathbf{h}^{l}\in \mathbb{R}^{N\times C^{l}}$ is the feature layer in the $l$-th layer and $C^{l}$ is the number of feature channels in the $l$-th layer. We initialize the input layer $\mathbf{h}^0$ with all ones and $\mathbf{h}^{l+1}$ is computed from the previous layer $\mathbf{h}^l$ with layer-wise convolutions:
\begin{equation}
\mathbf{h}^{l+1}=\sigma(\mathbf{h}^{l}\mathbf{U}_0^{l}+\mathbf{D}^{-\frac{1}{2}}\mathbf{A}\mathbf{D}^{-\frac{1}{2}}\mathbf{h}^{l}\mathbf{U}_1^{l}),
\end{equation}
where $\mathbf{U}_0^{l}\in \mathbb{R}^{C^{l}\times C^{l+1}}$ and $\mathbf{U}_1^{l}\in \mathbb{R}^{C^{l}\times C^{l+1}}$ are trainable weights in the convolutions of the network, $\mathbf{D}$ is the degree matrix of $\mathbf{A}$ with its diagonal entry $\mathbf{D}(i,i)=\sum_j\mathbf{A}(j,i)$, and $\sigma(\cdot)$ is the ReLU \citep{nair2010rectified} activation function. After the GCN backbone, a 10-layer Multi-Layer Perceptron (MLP) with residual connections \citep{he2016deep} is appended and generates the final continuous parameterization $\boldsymbol\theta$ for all the nodes.

\section{Experimental Details}

\subsection{TSP}

\paragraph{Training}
For TSP-500, we train our model for 120 meta-gradient descent steps (1.5\,h in total) with $T=15$. For TSP-1000, we train our model for 120 meta-gradient descent steps (1.7\,h in total) with $T=14$. For TSP-10000, we train our model for 50 meta-gradient descent steps (10\,h in total) with $T=12$. We generate 3 instances per meta-gradient descent step. We use the AdamW optimizer \citep{adamw} with learning rate $0.005$ and weight decay $0.0005$ for meta-gradient descent steps, and with learning rate $0.05$ for REINFORCE gradient descent steps. For other learning-based baseline methods, we download and rerun the source codes published by their original authors based on their pre-trained models.

\paragraph{Hardware}
We follow the hardware environment suggested by \citet{fu2020generalize}. For the three traditional algorithms, since their source codes do not support running on GPUs, they run on Intel Xeon Gold 5118 CPU @ 2.30GHz. To ensure fair comparison, learning-based methods run on GTX 1080 Ti GPU during the testing phase. MCTS runs on Intel Xeon Gold 6230 80-core CPU @ 2.10GHz, where we use 64 threads for TSP-500 and TSP-1000, and 16 threads for TSP-10000. For the training phase, we train our model on NVIDIA Tesla P100 16GB GPU.

\paragraph{Reproduction}
We implement \method{} for TSP based on PyTorch Geometric \cite{pyg} in LibTorch and PyTorch \citep{pytorch}. Our code for TSP is publicly available.\footnote{\url{https://github.com/DIMESTeam/DIMES} (MIT license)} The test instances are provided by \citet{fu2020generalize}.\footnote{\url{https://github.com/Spider-scnu/TSP} (MIT license)}

\subsection{MIS}

\paragraph{Training}

For SAT, we train our model for 50k meta-gradient steps with $T=1$. For ER-[700-800], we train our model for 150k meta-gradient steps with $T=1$. For ER-[9000-11000], we initialize our model from the checkpoint of ER-[700-800], and further train it for 200 meta-gradient steps. We use a batch size of $8$ on all datasets and Adam optimizer \citep{kingma2014adam} with learning rate $0.001$ for the meta-gradient descent step, and with learning rate 0.0002 for REINFORCE gradient descent steps. For other learning-based baseline methods, we mainly use an integrated implementation\footnote{\url{https://github.com/MaxiBoether/mis-benchmark-framework} (No license)} provided by \citet{other2022whats}.

\paragraph{Hardware}

All the methods are trained and evaluated on a single NVIDIA Ampere A100 40 GB GPU, with AMD EPYC 7713 64-Core CPUs.

\paragraph{Reproduction}

Our code for MIS is publicly available.\footnote{\url{https://github.com/DIMESTeam/DIMES} (MIT license)} Following \citet{other2022whats}, for SAT, we use the “Random-3-SAT Instances with Controlled Backbone Size” dataset\footnote{\url{https://www.cs.ubc.ca/~hoos/SATLIB/Benchmarks/SAT/CBS/descr_CBS.html}} and randomly split it into 39500 training instances and 500 test instances. For the Erd\H os-R\'enyi graphs, both training and test instances are randomly generated.

\section{Proofs}

\newtheorem{PRP}{Proposition}
\theoremstyle{remark}
\newtheorem*{REM}{Remark}
\newcommand\TSP{\textnormal{TSP}}
\newcommand\MIS{\textnormal{MIS}}
\newcommand\AL[1]{\begin{align}#1\end{align}}
\newcommand\EQ[1]{\begin{gather}#1\end{gather}}

In this section, we follow the notation introduced in Section~\ref{sec:method}.

\subsection{Convergence of Solution Distributions}

The following propositions show that $p_{\BM\theta}$ and $q_{\BM\theta}$ converge to the \emph{same} solution. They imply that we can optimize $q_{\BM\theta}$ instead of $p_{\BM\theta}$.

\begin{PRP}[TSP version]\label{prp:aux-tsp}
Let $0<\delta\ll1$ be a sufficiently small number. If $q_{\BM\theta}^\TSP(f)\ge1-\delta$ for a solution $f\in\CAL F$, then we also have $p_{\BM\theta}(f)\ge1-O(\delta)$.
\end{PRP}

\begin{PRP}[MIS version]\label{prp:aux-mis}
Suppose that $\BM\theta$ is normalized (i.e., $\sum_i\exp(\theta_i)=1$) and uniformly bounded w.r.t.\ a solution $f\in\CAL F$ (i.e., $\sum_if_i\exp(\theta_i)/\exp(\sum_if_i\theta_i)\le L$ for a constant $L>0$).
Let $0<\delta\ll1$ be a sufficiently small number. If $q_{\BM\theta}^\MIS(f)\ge1-\delta$, then we also have $p_{\BM\theta}(f)\ge1-O(\delta)$.
\end{PRP}

\begin{REM}
Propositions~\ref{prp:aux-tsp} \& \ref{prp:aux-mis} imply that if $q_{\BM\theta}$ converges to $f$ ($\delta\to0_+$), then $p_{\BM\theta}$ also converges to $f$.
\end{REM}

\begin{proof}[Proof for TSP]
Using the bound of $q_{\BM\theta}^\TSP(f)$, we have for any node $j$:
\AL{q_\TSP(\pi_f\mid\pi_f(0)=j)&=nq_{\BM\theta}^\TSP(f)-\sum_{i\ne j}q_\TSP(\pi_f\mid\pi_f(0)=i)\\
&\ge nq_{\BM\theta}^\TSP(f)-(n-1)\\
&\ge n(1-\delta)-(n-1)=1-O(\delta).}
Thus, for any edge $(i,j)$ in the tour $\pi_f$ and any edge $(i,k)\ne(i,j)$,
\AL{\theta_{i,j}-\theta_{i,k}&=\log\frac{\exp(\theta_{i,j})}{\exp(\theta_{i,k})}\\&\ge\log\frac{q_{\BM\theta}(\pi_f(1)=j\mid\pi_f(0)=i)}{1-q_{\BM\theta}(\pi_f(1)=j\mid\pi_f(0)=i)}\\
&\ge\log\frac{q_{\BM\theta}(\pi_f\mid\pi_f(0)=i)}{1-q_{\BM\theta}(\pi_f\mid\pi_f(0)=i)}\\
&\ge\log\frac{1-O(\delta)}{O(\delta)}.}
Note that for any edge $(i,j)$ in the tour $f$ (denoted by $(i,j)\in\pi_f$) and any solution $g\in\CAL F\setminus\{f\}$, there exist a unique $k_i^{g}$ such that edge $(i,k_i^{g})$ is in the tour $\pi_{g}$, and $(i,k_i^{g})\ne(i,j)$ for at least one edge $(i,j)\in\pi_f$. Then,
\AL{p_{\BM\theta}(f)&=\frac{1}{1+\sum_{g\in\CAL F\setminus\{f\}}\exp\big({-\sum_{(i,j)\in f}(\theta_{i,j}-\theta_{i,k_i^g})}\big)}\\
&=\frac{1}{1+\sum_{g\in\CAL F\setminus\{f\}}\exp\big({-\sum_{(i,j)\in f\setminus g}(\theta_{i,j}-\theta_{i,k_i^g})}\big)}\\
&\ge\frac{1}{1+\sum_{g\in\CAL F\setminus\{f\}}\exp\big({-\sum_{(i,j)\in f\setminus g}\log\frac{1-O(\delta)}{O(\delta)}}\big)}\\
&=1-O(\delta).}
\end{proof}

\begin{proof}[Proof for MIS]
Let $|g|$ denote the size of a solution $g\in\CAL F$, i.e., $|g|=\sum_ig_i$. With a little abuse of notation, let $g\in\CAL F$ also denote the corresponding independent set. Note that
\AL{&\frac{\max_{i\notin f}\exp(\theta_i)}{\max_{i\notin f}\exp(\theta_i)+\sum_{i\in f}\exp(\theta_i)}\\
\le{}&\frac{\sum_{i\notin f}\exp(\theta_i)}{\sum_{i\notin f}\exp(\theta_i)+\sum_{i\in f}\exp(\theta_i)}\\
={}&\frac{\sum_{i\notin f}\exp(\theta_i)}{\sum_i\exp(\theta_i)}=\sum_{i\notin f}q_\MIS(a_1=i)\\
={}&q_\MIS(a_1\notin f)\le1-q_{\BM\theta}^\MIS(f)\le\delta.}
This implies
\AL{\max_{i\notin f}\exp(\theta_i)\le\frac{\delta}{1-\delta}\sum_{i\in f}\exp(\theta_i).}
Recall that we have assumed in Section~3.2.2 that each $f'\in\mathcal F$ is not a proper subset of any other $f''\in\mathcal F$. Thus for any $f,g\in\mathcal F$, we have $f\setminus g\ne\varnothing$, and $g\setminus f\ne\varnothing$. Note also that $\exp(\theta_i)\le\sum_j\exp(\theta_j)=1$ for all nodes $i$. Hence,
\AL{p_{\BM\theta}(f)&=\bigg(1+\sum_{g\in\mathcal F\setminus\{f\}}\frac{\exp(\sum_ig_i\theta_i)}{\exp(\sum_if_i\theta_i)}\bigg)^{-1}\\
&=\bigg(1+\sum_{g\in\mathcal F\setminus\{f\}}\frac{\prod_{i\in g\setminus f}\exp(\theta_i)}{\prod_{i\in f\setminus g}\exp(\theta_i)}\bigg)^{-1}\\
&\ge\bigg(1+\sum_{g\in\mathcal F\setminus\{f\}}\frac{\max_{i\in g\setminus f}\exp(\theta_i)}{\prod_{i\in f\setminus g}\exp(\theta_i)}\bigg)^{-1}\\
&\ge\bigg(1+\sum_{g\in\mathcal F\setminus\{f\}}\frac{\max_{i\notin f}\exp(\theta_i)}{\prod_{i\in f}\exp(\theta_i)}\bigg)^{-1}\\
&\ge\bigg(1+\sum_{g\in\mathcal F\setminus\{f\}}\frac{\frac{\delta}{1-\delta}\sum_{i\in f}\exp(\theta_i)}{\prod_{i\in f}\exp(\theta_i)}\bigg)^{-1}\\
&\ge\bigg(1+\sum_{g\in\mathcal F\setminus\{f\}}\frac{\delta}{1-\delta}\cdot L\bigg)^{\!-1}\\
&=1-O(\delta).
}
\end{proof}

\subsection{First-Order Approximation of Meta-Gradient}

The following proposition gives a first-order approximation formula of the meta-gradient.

\begin{PRP}
Let $F_{\boldsymbol\varPhi}(\kappa_s,A_s)$ be a GNN $F$ with parameter $\BM\varPhi$ and input $(\kappa_s,A_s)$, $\CAL L(\BM\varPhi\mid\{s\})$ be a loss function, and $\alpha>0$ be a learning rate. Suppose $\boldsymbol{\varPhi}_s^{(0)}=\boldsymbol{\varPhi}$, and $\boldsymbol{\varPhi}_s^{(t)}=\boldsymbol{\varPhi}_s^{(t-1)} - \alpha\nabla_{\boldsymbol{\varPhi}_s^{(t-1)}} \mathcal{L}(\boldsymbol{\varPhi}_s^{(t-1)} \mid \{s\})$ for $1\le t \le T$, and $\BM\theta_s^{(T)}=F_{\boldsymbol\varPhi_s^{(T)}}(\kappa_s,A_s)$. Then,
\[\nabla_{\boldsymbol{\varPhi}}{\boldsymbol{\theta}^{(T)}_s}=\nabla_{\BM\varPhi_s^{(T)}}F_{\boldsymbol\varPhi_s^{(T)}}(\kappa_s,A_s)+O(\alpha).\]
\end{PRP}

\begin{proof}The proof resembles \cite{reptile}. By chain rule,
\begin{align}
\nabla_{\BM\varPhi_s^{(0)}}\BM\varPhi_s^{(T)}&=\prod_{t=1}^T\nabla_{\BM\varPhi_s^{(t-1)}}\BM\varPhi_s^{(t)}\\
&=\prod_{t=1}^T\nabla_{\BM\varPhi_s^{(t-1)}}(\BM\varPhi_s^{(t-1)}-\alpha\nabla_{\boldsymbol{\varPhi}_s^{(t-1)}}\mathcal{L}(\boldsymbol{\varPhi}_s^{(t-1)}\mid\{s\}))\\
&=\prod_{t=1}^T(\BM I-\alpha\nabla^2_{\boldsymbol{\varPhi}_s^{(t-1)}}\mathcal{L}(\boldsymbol{\varPhi}_s^{(t-1)}\mid\{s\}))\\
&=\BM I+\sum_{k=1}^T(-\alpha)^k\sum_{1\le t_1<\dots<t_k\le T}\prod_{i=1}^k\nabla^2_{\boldsymbol{\varPhi}_s^{(t_i-1)}}\mathcal{L}(\boldsymbol{\varPhi}_s^{(t_i-1)}\mid\{s\})\\
&=\BM I+O(\alpha).
\end{align}
Hence,
\begin{align}
\nabla_{\boldsymbol{\varPhi}}{\boldsymbol{\theta}^{(T)}_s}&=\nabla_{\BM\varPhi_s^{(0)}}\BM\varPhi_s^{(T)}\nabla_{\BM\varPhi_s^{(T)}}F_{\boldsymbol\varPhi_s^{(T)}}(\kappa_s,A_s)\\
&=(\BM I+O(\alpha))\nabla_{\BM\varPhi_s^{(T)}}F_{\boldsymbol\varPhi_s^{(T)}}(\kappa_s,A_s)\\
&=\nabla_{\BM\varPhi_s^{(T)}}F_{\boldsymbol\varPhi_s^{(T)}}(\kappa_s,A_s)+O(\alpha).
\end{align}
\end{proof}

\section{Additional Experiments for TSP}

\subsection{Performance on TSP-100}
We trained \method{} on TSP-100 and evaluate it on TSP-100 with $T=10$ and $0$ (i.e., with and without meta-learning). Since MCTS is the best per-instance search scheme for \method{} (see Table~\ref{tab:exp-500-1k-10k}), we also use MCTS here. When using AS, we fine-tune \method{} on each instance for 100 steps. We compare \method{} with learning-based methods listed in Section~\ref{sec:main}. Results of baselines are taken from \citet{fu2020generalize}. The results are presented in Table~\ref{tab:exp-100}. 

As is shown in the table, \method{} outperforms all learning-based methods, and its results are very close to optimal lengths given by exact solvers. The results suggest that \method{} achieves the best in-distribution performance among learning-based methods. Notably, with meta-learning ($T=10$), even when \method{} does not fine-tune (i.e., no active search) for each problem instance in evaluation, it still outperforms all other learning-based methods. This again demonstrates the efficacy of meta-learning to \method{}.

\begin{table}[t]\scriptsize
\caption{Results on TSP-100. * indicates the baseline for computing the performance drop.}
\label{tab:exp-100}
\begin{center}
\begin{tabular}{llcc}
\toprule
Method&Type&Length $\downarrow$&Drop $\downarrow$\\
\midrule
Concorde&OR (exact)&7.7609*&---\\
Gurobi  &OR (exact)&7.7609*&---\\
LKH-3   &OR&7.7611&0.0026\%\\
\midrule
EAN&RL+S&8.8372&13.8679\%\\
EAN&RL+S+2-OPT&8.2449&6.2365\%\\
AM&RL+S&7.9735&2.7391\%\\
AM&RL+G&8.1008&4.3791\%\\
AM&RL+BS&7.9536&2.4829\%\\
GCN&SL+G&8.4128&8.3995\%\\
GCN&SL+BS&7.8763&1.4828\%\\
Att-GCN&SL+MCTS&7.7638&0.0370\%\\
\midrule
\method{} ($T=0$)&RL+MCTS    &7.7647    &0.0490\%\\
\method{} ($T=0$)&RL+AS+MCTS &7.7618    &0.0116\%\\
\method{} ($T=10$)&RL+MCTS   &7.7620    &0.0142\%\\
\method{} ($T=10$)&RL+AS+MCTS&\textbf{7.7617}	&\textbf{0.0103\%}\\
\bottomrule
\end{tabular}
\end{center}
\end{table}

\begin{table}[t]\scriptsize
\caption{Results of \method{} (RL+S). ``Trained on TSP-100'' indicates extrapolation performance.}
\label{tab:exp-extrapolation}
\begin{center}
\begin{tabular}{l|cc|cc|cc}
\toprule
\multirow{2}*{Setting}
&\multicolumn{2}{c|}{TSP-500}&\multicolumn{2}{c|}{TSP-1000}&\multicolumn{2}{c}{TSP-10000}\\
&Length $\downarrow$&Drop $\downarrow$&Length $\downarrow$&Drop $\downarrow$&Length $\downarrow$&Drop $\downarrow$\\
\midrule
Trained on TSP-$n$ & 18.84 & 13.84\% & 26.36 & 14.01\% & 85.75 &19.48\%\\
Trained on TSP-100 & 19.21 & 16.07\% & 27.21 & 17.69\% & 86.24 &20.16\%\\
\bottomrule
\end{tabular}
\end{center}
\end{table}

\subsection{Extrapolation Performance}
We evaluate the exptrapolation performance of \method{} (i.e., trained on smaller graphs and tested on larger graphs). We train the model on TSP-100 and test it on TSP-500/1000/10000. For testing, we use RL+S ($\tau=0.01$) without active search. The results are reported in Table~\ref{tab:exp-extrapolation} in comparison with corresponding results trained on larger graphs (TSP-$n$).

From the table we can observe that the performance of \method{} does not drop much, which demonstrates the nice extrapolation performance of \method{}. One of our hypotheses is that graph sparsification in our neural network (see Appendix~\ref{apd:arch-tsp}) avoids the explosion of activation values in the graph neural network. Another hypothesis is that meta learning tends to not generate too extreme values in  (see point 9 of our previous response) and hence improve the generalization capability.

\begin{table}[t]
\scriptsize
\caption{Comparison of training settings for TSP-500/1000/10000.}
\label{tab:train-cost}
\begin{center}
\begin{tabular}{lccc}
\toprule
\textbf{Setting} & \textbf{AM} & \textbf{POMO} & \textbf{\method{}}\\
\midrule
Training problem scale&TSP-100&TSP-100&TSP-500\,/\,1000\,/\,10000\\
Training descent steps&250,000&312,600&120\,/\,120\,/\,50\\
Per-step training instances&512&64&3\\
Total training instances&128,000,000&20,000,000&360\,/\,360\,/\,150\\
Per-step training time&0.66\,s&0.28\,s&45\,s\,/\,51\,s\,/\,12\,m\\
Total training time&2\,d&1\,d&1.5\,h\,/\,1.7\,h\,/\,10\,h\\
Training GPUs&2&1&1\\
\bottomrule
\end{tabular}
\end{center}
\end{table}

\begin{figure}[t]
\begin{center}
\begin{subfigure}[b]{0.45\textwidth}
\centering
\includegraphics[width=\textwidth]{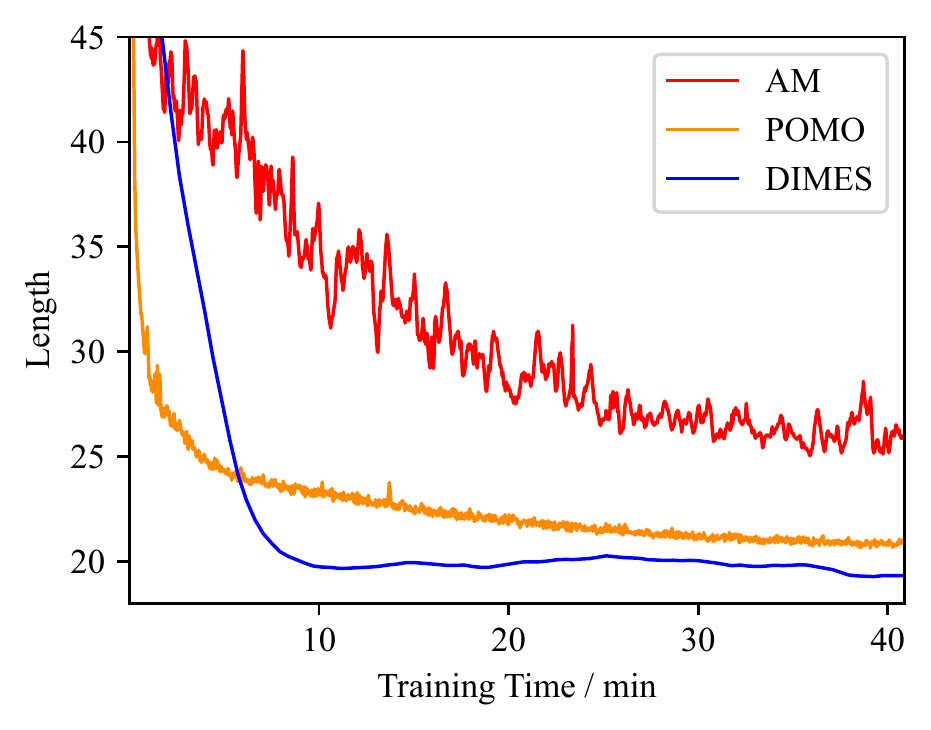}
\caption{Performance vs training time.}
\label{fig:tr-dyn-time}
\end{subfigure}
\hfill
\begin{subfigure}[b]{0.45\textwidth}
\centering
\includegraphics[width=\textwidth]{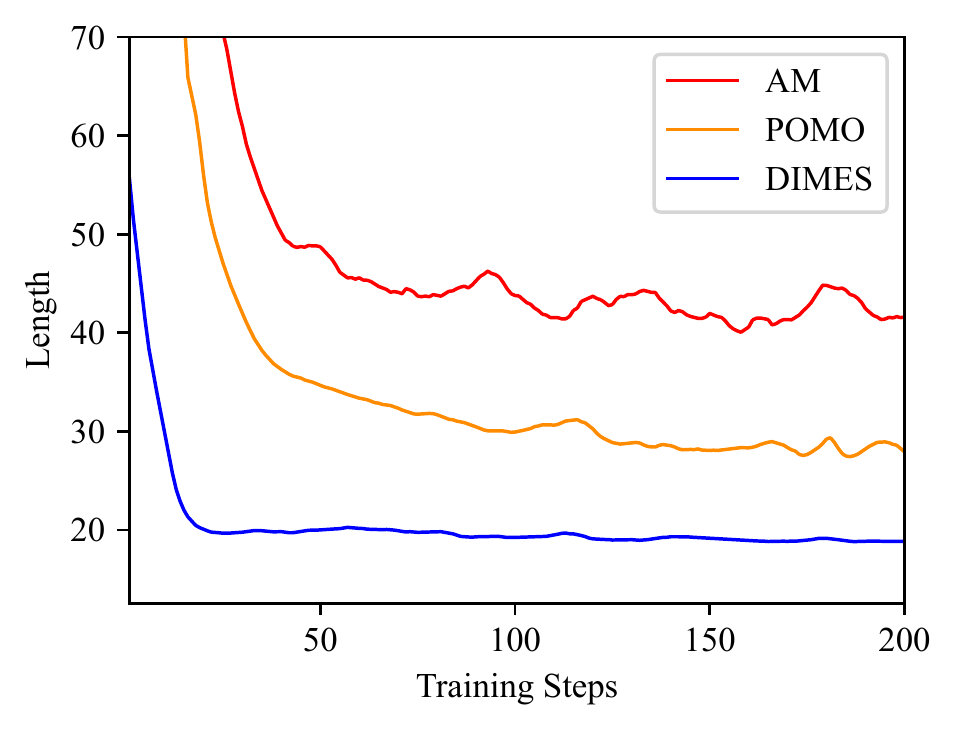}
\caption{Performance vs training steps.}
\label{fig:tr-dyn-steps}
\end{subfigure}
\end{center}
\caption{Evaluation performance vs training cost.}
\label{fig:tr-dyn}
\end{figure}

\subsection{Stability of Training}
We compare the training settings of AM \cite{19iclr-am}, POMO \cite{kwon2020pomo}, and \method{} in Table~\ref{tab:train-cost}. The training costs of AM and POMO are obtained from their papers\footnote{For AM/POMO, per-step training time is estimated by total training time divided by total training steps.} A training step means a gradient descent step of the GNN. That is, for AM/POMO, a training step means a gradient descent step over a batch; for DIMES, a training step means a meta-gradient descent step.

The table shows that \method{} is much more sample-efficient than AM/POMO. Notably, \method{} achieves stable training using only 3 instances per meta-gradient descent step. Hence, its total training time is accordingly much shorter, even though its per-step time is longer. Moreover, the stability of training enables us to use a larger learning rate, which also accelerates training.

To further illustrate the fast stable training of \method{}, we compare the dynamics of training among AM, POMO, and \method{} in Figure~\ref{fig:tr-dyn}. We closely follow the training settings of their papers, i.e., we train AM/POMO on TSP-100 and \method{} on TSP-500. For AM/POMO, we train their models on our hardware by re-running their public source code. The performance is evaluated using TSP-500 test instances. For DIMES, we use RL+S in evaluation.

From Figure~\ref{fig:tr-dyn-time}, we can observe that \method{} stably converges to a better performance within fewer time, while the dynamics of training AM/POMO are slower and less stable. From Figure~\ref{fig:tr-dyn-steps}, we can observe that \method{} converges at much fewer training steps. The results again demonstrate that the training of \method{} is fast and stable.